\documentclass{article}

\usepackage[numbers]{natbib}
\usepackage[utf8]{inputenc} %
\usepackage[T1]{fontenc}    %
\usepackage{hyperref}       %
\usepackage{url}            %
\usepackage{booktabs}       %
\usepackage{amsfonts}       %
\usepackage{nicefrac}       %
\usepackage{microtype}      %

\usepackage{notation}
\usepackage{amsthm,amssymb}
\usepackage[algo2e,ruled]{algorithm2e}
\usepackage{algorithm,algcompatible}
\usepackage{OCO}
\usepackage{subfigure}
\usepackage{bibspacing}

\DeclareRobustCommand{\VAN}[3]{#2} %

\usepackage[textwidth=2.75cm,shadow,color=yellow,disable]{todonotes}

\newtheorem{theorem}{Theorem}
\newtheorem{lemma}[theorem]{Lemma}

\newcommand{\Mfull}{\M^\textnormal{full}}
\newcommand{\Mdiag}{\M^\textnormal{diag}}
\newcommand{\Dfull}{D^\textnormal{full}}
\newcommand{\Ddiag}{D^\textnormal{diag}}
\newcommand{\Gfull}{G^\textnormal{full}}
\newcommand{\Gdiag}{G^\textnormal{diag}}

\newcommand{\alphafull}{\alpha^\textnormal{full}}
\newcommand{\alphadiag}{\alpha^\textnormal{diag}}
\newcommand{\rtrick}{\tilde{r}}     %
\newcommand{\Rtrick}{\tilde{R}}     %
\DeclareBoldMathCommand{\X}{X}

\newcommand{\citeauthornumber}[2][]{\citeauthor*{#2} \cite[#1]{#2}}

\renewcommand{\U}{\mathcal U}
\renewcommand{\P}{\mathcal P}   %

\let\origlog\log    %
\let\log\ln %
\let\Tr\tr  %

\newcommand{\lambdamax}{\lambda_\text{max}}

\DeclareMathOperator{\sign}{sign}

\title{MetaGrad: Multiple Learning Rates \linebreak in Online Learning}

\author{
  Tim van Erven\\
  Leiden University\\
  \texttt{tim@timvanerven.nl}
  \and
  Wouter M. Koolen\\
  Centrum Wiskunde \& Informatica\\
  \texttt{wmkoolen@cwi.nl}\\
}

\begin{document}

\maketitle

\begin{abstract}
In online convex optimization it is well known that certain subclasses
of objective functions are much easier than arbitrary convex functions.
We are interested in designing adaptive methods that can automatically
get fast rates in as many such subclasses as possible, without any manual
tuning. Previous adaptive methods are able to interpolate between
strongly convex and general convex functions. We present a new method,
MetaGrad, that adapts to a much broader class of functions, including
exp-concave and strongly convex functions, but also various types of
stochastic and non-stochastic functions without any curvature. For
instance, MetaGrad can achieve logarithmic regret on the unregularized
hinge loss, even though it has no curvature, if the data come from a
favourable probability distribution.
MetaGrad's main feature is that it simultaneously considers multiple
learning rates. Unlike previous methods with provable
regret guarantees, however, its learning rates are not monotonically
decreasing over time and are not tuned based on a theoretically derived
bound on the regret. Instead, they are weighted directly proportional to
their empirical performance on the data using a tilted exponential
weights master algorithm.
\end{abstract}

\section{Introduction}

Methods for \emph{online convex optimization} (OCO)
\citep{ShalevShwartz2012,Hazan2016} make it possible to optimize
parameters sequentially, by processing convex functions in a streaming
fashion. This is important in time series prediction where the data are
inherently online; but it may also be convenient to process offline data
sets sequentially, for instance if the data do not all fit into memory
at the same time or if parameters need to be updated quickly when extra
data become available.

The difficulty of an OCO task depends on the convex functions
$f_1,f_2,\ldots,f_T$ that need to be optimized. The argument of these
functions is a $d$-dimensional parameter vector $\w$ from a convex
domain $\U$. Although this is abstracted away in the general framework,
each function $f_t$ usually measures the loss of the parameters on an
underlying example $(\x_t,y_t)$ in a machine learning task. For example,
in classification $f_t$ might be the \emph{hinge loss} $f_t(\w) =
\max\{0,1-y_t \ip{\w}{\x_t}\}$ or the \emph{logistic loss} $f_t(\w) =
\log\del*{1 + e^{-y_t \ip{\w}{\x_t}}}$, with $y_t \in \{-1,+1\}$. Thus
the difficulty depends both on the choice of loss and on the observed
data.

There are different methods for OCO, depending on assumptions that can
be made about the functions. The simplest and most commonly used
strategy is \emph{online gradient descent} (GD), which does not require
any assumptions beyond convexity. GD updates parameters $\w_{t+1} = \w_t
- \eta_t \nabla f_t(\w_t)$ by taking a step in the direction of the
negative gradient, where the step size is determined by a parameter
$\eta_t$ called the \emph{learning rate}. For learning rates $\eta_t
\propto 1/\sqrt{t}$, GD guarantees that the \emph{regret} over $T$
rounds, which measures the difference in cumulative loss between the
online iterates $\w_t$ and the best offline parameters $\u$, is bounded
by $O(\sqrt{T})$ \citep{Zinkevich2003}. Alternatively, if it is known
beforehand that the functions are of an easier type, then better regret
rates are sometimes possible. For instance, if the functions are
\emph{strongly convex}, then logarithmic regret $O(\log T)$ can be
achieved by GD with much smaller learning rates $\eta_t \propto 1/t$
\citep{ons}, and, if they are \emph{exp-concave}, then logarithmic
regret $O(d \log T)$ can be achieved by the \emph{Online Newton Step}
(ONS) algorithm \citep{ons}.

This partitions OCO tasks into categories, leaving it to the user to
choose the appropriate algorithm for their setting. Such a strict
partition, apart from being a burden on the user, depends on an
extensive cataloguing of all types of easier functions that might occur
in practice. (See Section~\ref{sec:fastRateExamples} for several ways in
which the existing list of easy functions can be extended.) It also
immediately raises the question of whether there are cases in between
logarithmic and square-root regret (there are, see
Theorem~\ref{thm:Bernstein} in Section~\ref{sec:fastRateExamples}), and
which algorithm to use then. And, third, it presents the problem that
the appropriate algorithm might depend on (the distribution of) the data
(again see Section~\ref{sec:fastRateExamples}), which makes it entirely
impossible to select the right algorithm beforehand. 

These issues motivate the development of \emph{adaptive} methods, which
are no worse than $O(\sqrt{T})$ for general convex functions, but also
automatically take advantage of easier functions whenever possible. An
important step in this direction are the adaptive GD algorithm of
\citeauthornumber{BartlettHazanRakhlin2007} and its proximal improvement by
\citeauthornumber{Do2009}, which are able to interpolate between strongly convex
and general convex functions if they are provided with a data-dependent
strong convexity parameter in each round, and significantly outperform
the main non-adaptive method (i.e.\ Pegasos,
\citep{Shalev-ShwartzEtAl2011Pegasos}) in
the experiments of \citeauthor{Do2009} Here we consider a significantly richer
class of functions, which includes exp-concave functions, strongly
convex functions, general convex functions that do not change between
rounds (even if they have no curvature), and stochastic functions whose
gradients satisfy the so-called Bernstein condition, which is well-known
to enable fast rates in offline statistical learning
\citep{BartlettMendelson2006,VanErven2015FastRates,AndereNIPSpaper2016}.
The latter group can again include functions without curvature, like the
unregularized hinge loss. All these cases are covered simultaneously by
a new adaptive method we call \emph{MetaGrad}, for \underbar{m}ultiple
\underbar{eta} \underbar{grad}ient algorithm. MetaGrad maintains a
covariance matrix of size $d \times d$ where $d$ is the parameter
dimension. In the remainder of the paper we call this version \emph{full
MetaGrad}. A reference implementation is available from~\cite{MetaGradCode}. We also design and analyze a faster approximation that only
maintains the $d$ diagonal elements, called \emph{diagonal MetaGrad}.
Theorem~\ref{thm:mainbound} below implies the following:
\begin{theorem}\label{thm:roughthm}
Let $\grad_t = \nabla f_t(\w_t)$ and $V_T^\u = \sum_{t=1}^T \del*{(\u - \w_t)^\top \grad_t}^2$. Then the regret of full MetaGrad is
simultaneously bounded by $O(\sqrt{T \log \log T})$, and by
\begin{equation}\label{eqn:roughmainbound}
\sum_{t=1}^T f(\w_t) - \sum_{t=1}^T f_t(\u)
~\le~
\sum_{t=1}^T (\w_t - \u)^\top \grad_t
~\le~
O\del*{
\sqrt{
  V_T^\u\,
  d \ln T
}
+ d \ln T
}
\end{equation}
for any $\u \in \U$.
\end{theorem}
Theorem~\ref{thm:roughthm} bounds the regret in terms of a measure of variance
$V_T^\u$ that depends on the distance of the algorithm's choices $\w_t$
to the optimum $\u$, and which, in favourable cases, may be
significantly smaller than $T$. Intuitively, this happens, for instance,
when there is stable optimum $\u$ that the algorithm's choices $\w_t$
converge to. Formal consequences are given in
Section~\ref{sec:fastRateExamples}, which shows that this bound implies
faster than $O(\sqrt{T})$ regret rates, often logarithmic in $T$, for
all functions in the rich class mentioned above. In all cases the
dependence on $T$ in the rates matches what we would expect based on
related work in the literature, and in most cases the dependence on the
dimension $d$ is also what we would expect. Only for strongly convex
functions is there an extra factor $d$. It is an open question whether
this is a fundamental obstacle for which an even more general adaptive
method is needed, or whether it is an artefact of our analysis.

The main difficulty in achieving the regret guarantee from
Theorem~\ref{thm:roughthm} is tuning a learning rate parameter $\eta$.
In theory, $\eta$ should be roughly $1/\sqrt{V_T^\u}$, but this is not
possible using any existing techniques, because the optimum $\u$ is unknown in
advance, and tuning in terms of a uniform upper bound $\max_\u V_T^\u$
ruins all desired benefits. MetaGrad therefore runs multiple slave
algorithms, each with a different learning rate, and combines them with
a novel master algorithm that learns the empirically best learning rate
for the OCO task in hand. The slaves are instances of exponential
weights on the continuous parameters $\u$ with a suitable surrogate loss
function, which in particular causes the exponential weights
distributions to be multivariate Gaussians. For the full version of
MetaGrad, the slaves are closely related to the ONS algorithm on the original
losses, where each slave receives the master's gradients instead
of its own. It is shown that $\ceil{\half \origlog_2 T} + 1$ slaves
suffice, which is at most $16$ as long as $T \leq 10^9$, and therefore
seems computationally acceptable. If not, then the number of slaves can
be further reduced at the cost of slightly worse constants in the bound.

\paragraph{Related Work}

If we disregard computational efficiency, then the result of
Theorem~\ref{thm:roughthm} can be achieved by finely discretizing the
domain $\U$ and running the Squint algorithm for prediction with experts
with each discretization point as an expert \citep{squint}. MetaGrad may
therefore also be seen as a computationally efficient extension of
Squint to the OCO setting.

Our focus in this work is on adapting to sequences of functions $f_t$
that are easier than general convex functions. A different direction in
which faster rates are possible is by adapting to the domain $\U$. As we
assume $\U$ to be fixed, we consider an upper bound $D$ on the norm of the
optimum $\u$ to be known. In contrast,
\citeauthor*{OrabonaPal2016} \cite{Orabona2014,OrabonaPal2016} design methods that can
adapt to the norm of $\u$. One may also look at the shape of $\U$. As
can be seen in the analysis of the slaves, MetaGrad is based a spherical
Gaussian prior on $\reals^d$, which favours $\u$ with small
$\ell_2$-norm. This is appropriate for $\U$ that are similar to the
Euclidean ball, but less so if $\U$ is more like a box
($\ell_\infty$-ball). In this case, it would be better to run a copy of
MetaGrad for each dimension separately, similarly to how the diagonal
version of the AdaGrad algorithm \citep{adagrad,McMahanStreeter2010} may
be interpreted as running a separate copy of GD with a separate learning
rate for each dimension. AdaGrad further uses an adaptive tuning of the
learning rates that is able to take advantage of sparse gradient
vectors, as can happen on data with rarely observed features. We briefly
compare to AdaGrad in some very simple simulations in
Appendix~\ref{app:simulations}.

Another notion of adaptivity is explored in a series of work
\cite{hazan2010extracting,GradualVariationInCosts2012,SteinhardtLiang14}
obtaining tighter bounds for
linear functions $f_t$ that vary little between rounds (as measured
either by their deviation from the mean function or by successive
differences). Such bounds imply super fast rates for optimizing a fixed
linear function, but reduce to slow $O(\sqrt{T})$ rates in the other
cases of easy functions that we consider.
Finally, the way MetaGrad's slaves maintain a Gaussian distribution on
parameters $\u$ is similar in spirit to AROW and related confidence
weighted methods, as analyzed by \citeauthornumber{CrammerEtAl2009AROW}
in the mistake bound model.

\paragraph{Outline}

We start with the main definitions in the next section. Then
Section~\ref{sec:fastRateExamples} contains an extensive set of
examples where Theorem~\ref{thm:roughthm} leads to fast rates,
Section~\ref{sec:metagrad} presents the MetaGrad algorithm, and
Section~\ref{sec:analysis} provides the analysis leading to
Theorem~\ref{thm:mainbound}, which is a more detailed statement of
Theorem~\ref{thm:roughthm} with an improved dependence on the dimension
in some particular cases and with exact constants. The details of the
proofs can be found in the appendix.

\section{Setup}

\begin{algorithm2e}[t]
\renewcommand{\algorithmicrequire}{\textbf{Input:}}
\renewcommand{\algorithmiccomment}[1]{\hfill $\triangleright$~\textit{#1}}
\begin{algorithmic}[1]
\REQUIRE Convex set $\U$
\FOR{$t=1,2,\ldots$}
  \STATE Learner plays $\w_t \in \U$
  \STATE Environment reveals convex loss function $f_t : \U \to \reals$
  \STATE Learner incurs loss $f_t(\w_t)$ and observes (sub)gradient $\grad_t = \nabla f_t(\w_t)$
\ENDFOR
\end{algorithmic}
\SetAlgorithmName{Protocol}{Protocol}{List of Protocols}
\caption{Online Convex Optimization from First-order Information}\label{alg:OCOprotocol}
\end{algorithm2e}
\setcounter{algocf}{0}  %

Let $\U \subseteq \reals^d$ be a closed convex set, which we assume 
contains the origin $\zeros$ (if not, it can always be translated). We
consider algorithms for Online Convex Optimization over $\U$, which
operate according to the protocol displayed in
Protocol~\ref{alg:OCOprotocol}. Let $\w_t \in \U$ be the  iterate
produced by the algorithm in round $t$, let  $f_t : \U \to \reals$ be the
convex loss function produced by the environment and let $\grad_t =
\nabla f_t(\w_t)$ be the (sub)gradient, which is the feedback given to the
algorithm.\footnote{If $f_t$ is not differentiable at $\w_t$, any choice of subgradient $\grad_t \in \partial
f_t(\w_t)$ is allowed.} We abbreviate the
\emph{regret} with respect to $\u \in \U$ as $R_T^\u = \sum_{t=1}^T
\del*{f_t(\w_t) - f_t(\u)}$, and define our measure of variance as
$V_T^\u = \sum_{t=1}^T \del*{(\u - \w_t)^\top \grad_t}^2$ for the full
version of MetaGrad and $V_T^\u = \sum_{t=1}^T \sum_{i=1}^d (u_i -
w_{t,i})^2 \grads_{t,i}^2$ for the diagonal version. By convexity of
$f_t$, we always have $f_t(\w_t) - f_t(\u) \leq (\w_t - \u)^\top
\grad_t$. Defining $\Rtrick_T^\u =
\sum_{t=1}^T(\w_t - \u)^\top \grad_t$, this
implies the first inequality in
Theorem~\ref{thm:roughthm}: $R_T^\u \leq \Rtrick_T^\u$. A stronger requirement than
convexity is that a function $f$ is \emph{exp-concave}, which (for
exp-concavity parameter $1$) means that $e^{-f}$ is concave.
Finally, we impose the following standard boundedness assumptions,
distinguishing between the full version of MetaGrad (left column) and
the diagonal version (right column): for all $\u, \v \in \U$, all
dimensions $i$ and all times $t$,
\begin{align}
\notag
   & \text{full} & & \text{diag}
\\
\label{eq:B}
  \norm{\u - \v} &~\leq~ \Dfull
    & |u_i - v_i| &~\leq~ \Ddiag
\\
\notag
  \norm{\grad_t} &~\leq~ \Gfull
    & |\grads_{t,i}| &~\leq~ \Gdiag.
\end{align}
Here, and throughout the paper, the norm of a vector (e.g.\
$\|\grad_t\|$) will always refer to the $\ell_2$-norm. For the full
version of MetaGrad, the Cauchy-Schwarz inequality further implies that
$(\u - \v)^\top \grad_t \leq \|\u - \v\| \cdot \|\grad_t\| \leq \Dfull
\Gfull$. 

\section{Fast Rate Examples}\label{sec:fastRateExamples}

In this section, we motivate our interest in the adaptive bound
\eqref{eqn:roughmainbound} by giving a series of examples in which it
provides fast rates. These fast rates are all derived from two general
sufficient conditions: one based on the directional derivative of the
functions $f_t$ and one for stochastic gradients that satisfy the
\emph{Bernstein condition}, which is the standard condition for fast
rates in off-line statistical learning. Simple simulations that
illustrate the conditions are provided in Appendix~\ref{app:simulations}
and proofs are also postponed to
Appendix~\ref{app:MoreFastRateExamplesAndProofs}.

\paragraph{Directional Derivative Condition}

In order to control the regret with respect to some point $\u$, the
first condition requires a quadratic lower bound on the curvature of the
functions $f_t$ in the direction of $\u$:
\begin{theorem}\label{thm:curvedfunctions}
Suppose, for a given $\u \in \U$, there exist constants $a,b > 0$ such
that the functions $f_t$ all satisfy
\begin{equation}\label{eqn:curvedfunctions}
  f_t(\u) \geq f_t(\w) + a (\u - \w)^\top \nabla f_t(\w) + b \del*{(\u - \w)^\top \nabla f_t(\w)}^2
  \qquad \text{for all $\w \in \U$.}
\end{equation}
Then any method with regret bound \eqref{eqn:roughmainbound} incurs
logarithmic regret, $R_T^\u = O(d \ln T)$, with respect to $\u$.
\end{theorem}

The case $a=1$ of this condition was introduced by \citeauthornumber{ons}, who show
that it is satisfied for all $\u \in \U$ by exp-concave and strongly
convex functions. The rate $O(d \log T)$ is also what we would expect by
summing the asymptotic offline rate obtained by ridge regression on the
squared loss \citep[Section~5.2]{SrebroEtAl2010}, which is exp-concave.
Our extension to $a > 1$ is technically a minor step, but it makes the
condition much more liberal, because it may then also be satisfied by
functions that do \emph{not} have any curvature. For example, suppose
that $f_t = f$ is a fixed convex function that does not change with $t$.
Then, when $\u^* = \argmin_\u f(\u)$ is the offline minimizer, we have
$(\u^* - \w)^\top \nabla f(\w) \in \intcc{-\Gfull \Dfull,0}$, so that
\begin{align*}
  f(\u^*) - f(\w)
    &\geq (\u^* - \w)^\top \nabla f(\w)
\\
    &\geq 2 (\u^* - \w)^\top \nabla f(\w) + \frac{1}{\Dfull \Gfull} \del*{(\u^* - \w)^\top
    \nabla f(\w)}^2,
\end{align*}
where the first inequality uses only convexity of $f$. Thus condition
\eqref{eqn:curvedfunctions} is satisfied by \emph{any fixed convex
function}, even if it does not have any curvature at all, with $a =
2$ and $b=1/(\Gfull \Dfull)$.

\paragraph{Bernstein Stochastic Gradients}

The possibility of getting fast rates even without any curvature is
intriguing, because it goes beyond the usual strong convexity or
exp-concavity conditions. In the online setting, the case of fixed
functions $f_t = f$ seems rather restricted, however, and may in fact be
handled by offline optimization methods. We therefore seek to loosen
this requirement by replacing it by a stochastic condition on the
distribution of the functions $f_t$. The relation between variance
bounds like Theorem~\ref{thm:roughthm} and fast rates in the stochastic
setting is studied in depth by \citeauthornumber{AndereNIPSpaper2016}, who obtain
fast rate results both in expectation and in probability. Here we
provide a direct proof only for the expected regret, which allows a
simplified analysis.

Suppose the functions $f_t$ are independent and identically
distributed (i.i.d.), with common distribution $\pr$. Then we say that
the gradients satisfy the \emph{$(B,\beta)$-Bernstein condition} with
respect to the stochastic optimum 
\[
\u^* = \argmin_{\u \in \U} \E_{f \sim
\pr}[f(\u)]
\]
if, for all $\w \in \U$,
\begin{equation}\label{eqn:bernstein}
(\w - \u^*)^\top
\ex_f \sbr*{
  \nabla f(\w) \nabla f(\w)^\top
}
(\w - \u^*)
~\le~
B
\big((\w - \u^*)^\top  \ex_f \sbr*{\nabla f(\w)}\big)^\beta
.
\end{equation}
This is an instance of the well-known Bernstein condition from offline
statistical learning
\citep{BartlettMendelson2006,VanErven2015FastRates}, applied to the
linearized excess loss $(\w - \u^*)^\top \nabla f(\w)$.
As shown in Appendix~\ref{sec:bnst}, imposing the condition for the
linearized excess loss is a weaker requirement than imposing it for the
original excess loss $f(\w) - f(\u^*)$.

\begin{theorem}\label{thm:Bernstein}
If the gradients satisfy the $(B,\beta)$-Bernstein condition for $B > 0$
and $\beta \in (0,1]$ with respect to $\u^* = \argmin_{\u \in \U} \E_{f
\sim \pr}[f(\u)]$, then any method with regret bound
\eqref{eqn:roughmainbound} incurs expected regret
\[
\E\sbr{R_T^{\u^*}} =
O\del*{\del*{B d \ln T}^{1/(2-\beta)} T^{(1-\beta)/(2-\beta)}
    + d\ln T}.
\]
\end{theorem}
\noindent
For $\beta=1$, the rate becomes $O(d
\ln T)$, just like for fixed functions, and for smaller $\beta$ it is in
between logarithmic and $O(\sqrt{d T})$.
For instance, the hinge loss on the unit ball with i.i.d.\ data satisfies the Bernstein condition with $\beta = 1$, which implies
an $O(d \log T)$ rate. (See Appendix~\ref{app:hingeLossExample}.) It is
common to add $\ell_2$-regularization to the hinge loss to make it
strongly convex, but this example shows that that is not necessary
to get logarithmic regret.

\section{MetaGrad Algorithm}\label{sec:metagrad}

In this section we explain the two versions (full and diagonal) of the
MetaGrad algorithm. We will make use of the following definitions:
\begin{align}
\notag
   & \text{full} & & \text{diag}
\\
\label{eq:M.choices}
  \Mfull_t &~\df~ \grad_t \grad_t^\top
    & \Mdiag_t &~\df~ \diag(\grads_{t,1}^2, \ldots, \grads_{t,d}^2)
\\
\notag
  \alphafull &~\df~ 1
    & \alphadiag &~\df~ 1/d.
\end{align}
Depending on context, $\w_t \in \U$ will refer to the full or diagonal
MetaGrad prediction in round $t$. In the remainder we will drop the
superscript from the letters above, which will always be clear from
context.

MetaGrad will be defined by means of the following \emph{surrogate loss}
$\surr_t^\eta(\u)$, which depends on a parameter $\eta > 0$ that trades off \emph{regret} compared to $\u$ with
the square of the scaled directional derivative towards $\u$ (full case)
or its approximation (diag case):
\begin{equation}\label{eq:surrogate}
\surr_t^\eta(\u)
~\df~
- \eta(\w_t-\u)^\top \grad_t
+ \eta^2  (\u - \w_t)^\top \M_t  (\u - \w_t)
.
\end{equation}
Our surrogate loss consists of a linear and a quadratic part. 
Using the
language of \citeauthornumber{Orabona2015}, the data-dependent quadratic part
causes a ``time-varying regularizer'' and \citeauthornumber{adagrad} would call
it ``temporal adaptation of the proximal function''.
The
sum of quadratic terms in our surrogate is what appears in the regret
bound of Theorem~\ref{thm:roughthm}.

The MetaGrad algorithm is a two-level hierarchical construction,
displayed as Algorithms~\ref{alg:MetaGradMaster} (master algorithm that
learns the learning rate) and~\ref{alg:MetaGradSlave} (sub-module, a
copy running for each learning rate $\eta$ from a finite grid). Based on our
analysis in the next section, we recommend using the grid in
\eqref{eqn:grid}.

\begin{algorithm2e}[t]
\renewcommand{\algorithmicrequire}{\textbf{Input:}}
\renewcommand{\algorithmiccomment}[1]{\hfill $\triangleright$~\textit{#1}~~~~~~}%\renewcommand{\algorithmiccomment}[1]{\hfill $\triangleright$~\textit{#1}}
\begin{algorithmic}[1]
\REQUIRE Grid of learning rates $\frac{1}{5 D G} \geq \eta_1 \ge \eta_2 \ge \ldots$ with prior weights $\pi_1^{\eta_1}, \pi_1^{\eta_2}, \ldots$ \COMMENT{As in \eqref{eqn:grid}}
\FOR{$t=1,2,\ldots$}
\STATE Get prediction $\w_t^\eta \in \U$ of slave (Algorithm~\ref{alg:MetaGradSlave}) for each $\eta$
\STATE\label{line:tilted.ewa}
Play
$
\w_t
=
\frac{
  \sum_\eta \pi_t^\eta \eta \w^\eta_t
}{
  \sum_\eta \pi_t^\eta \eta \phantom{\w^\eta_t}
}
\in \U
$
\COMMENT{Tilted Exponentially Weighted Average}
\STATE\label{lin:gradien.trick}
Observe gradient $\grad_t = \nabla f_t(\w_t)$
\STATE\label{line:expw}
Update $\pi_{t+1}^\eta = \frac{\pi_t^\eta e^{-\alpha
\surr_t^\eta(\w_t^\eta)}}{\sum_\eta \pi_t^\eta e^{-\alpha\surr_t^\eta(\w_t^\eta)}}$ for all $\eta$
%\COMMENT{Exponential Weights with surrogate loss \eqref{eq:surrogate}}
\COMMENT{\parbox[t]{\widthof{Exponential Weights with}}{Exponential Weights with surrogate loss \eqref{eq:surrogate}}}
\ENDFOR
\end{algorithmic}
\caption{MetaGrad Master}\label{alg:MetaGradMaster}
\end{algorithm2e}

\begin{algorithm2e}
\renewcommand{\algorithmicrequire}{\textbf{Input:}}
\renewcommand{\algorithmiccomment}[1]{\hfill $\triangleright$~\textit{#1}~~~~~~}%
\begin{algorithmic}[1]
\REQUIRE Learning rate $0 < \eta \leq \frac{1}{5 D G}$, domain size $D > 0$
\STATE $\w^\eta_1 = \zeros$ and $\Sigma^\eta_1 = D^2 \I$
\FOR{$t=1,2,\ldots$}
\STATE Issue $\w_t^\eta$ to master (Algorithm~\ref{alg:MetaGradMaster})
\STATE Observe gradient $\grad_t = \nabla f_t(\w_t)$ \COMMENT{Gradient at \emph{master} point $\w_t$}
\STATE\label{line:md}
Update 
$
\begin{aligned}[t]
\Sigma^\eta_{t+1} &= \textstyle \del*{\frac{1}{D^2} \I + 2 \eta^2 \sum_{s=1}^t \M_s}^{-1}
\\
\widetilde{\w}_{t+1}^\eta &=  \w^\eta_t -  \Sigma^\eta_{t+1} \del*{\eta \grad_t + 2 \eta^2 \M_t (\w_t^\eta - \w_t)}
\\
\w_{t+1}^\eta &= \Pi_{\U}^{\Sigma_{t+1}^\eta} \del*{ \widetilde{\w}^\eta_{t+1}}
\end{aligned}$
\vspace{.3em}\linebreak
\mbox{}~~~with projection $\Pi_\U^{\Sigma}(\w) = \argmin_{\u \in \U} (\u - \w)^\top
\Sigma^{-1} (\u - \w)$
\ENDFOR
\end{algorithmic}

\smallskip
Implementation: For $\M_t = \Mdiag_t$ only maintain diagonal of $\Sigma_t^\eta$.
For $\M_t = \Mfull_t$ use rank-one update  $\Sigma^\eta_{t+1} =
\Sigma^\eta_t -  \frac{2 \eta^2\Sigma_t^\eta \grad_t \grad_t^\top
\Sigma_t^\eta}{1 + 2 \eta^2 \grad_t^\top \Sigma_t^\eta \grad_t}$ and simplify
$\widetilde{\w}_{t+1}^\eta =  \w^\eta_t -  \eta \Sigma^\eta_{t+1} \grad_t \del*{1 + 2 \eta\grad_t^\top (\w_t^\eta - \w_t)}$.

\caption{MetaGrad Slave}\label{alg:MetaGradSlave}
\end{algorithm2e}

\paragraph{Master}

The task of the Master Algorithm~\ref{alg:MetaGradMaster} is to learn
the empirically best learning rate $\eta$ (parameter of the surrogate
loss $\surr_t^\eta$), which is notoriously difficult to
track online because the regret is non-monotonic over rounds and may
have multiple local minima as a function of $\eta$ (see
\citep{learning.learning.rate} for a study in the expert setting). The
standard technique is therefore to derive a monotonic upper bound on the
regret and tune the learning rate optimally \emph{for the bound}. In
contrast, our approach, inspired by the approach for combinatorial games
of \citeauthornumber[Section~4]{squint}, is to have our master aggregate the
predictions of a discrete grid of learning rates. Although we provide a
formal analysis of the regret, the master algorithm does not depend on
the outcome of this analysis, so any slack in our bounds does not feed
back into the algorithm. The master is in fact very similar to the
well-known exponential weights method (line~\ref{line:expw}), run on the
surrogate losses, except that in the predictions the weights of the
slaves are \emph{tilted} by their learning rates
(line~\ref{line:tilted.ewa}), having the effect of giving a larger
weight to larger $\eta$. The internal parameter $\alpha$ is set to
$\alphafull$ from \eqref{eq:M.choices} for the full version of the
algorithm, and to $\alphadiag$ for the diagonal version.

\paragraph{Slaves}

The role of the Slave Algorithm~\ref{alg:MetaGradSlave} is to guarantee
small surrogate regret for a fixed learning rate $\eta$. We consider two
versions, corresponding to whether we take rank-one or diagonal matrices
$\M_t$ (see \eqref{eq:M.choices}) in the surrogate \eqref{eq:surrogate}.
The first version maintains a \emph{full} $d \times d$ covariance matrix
and has the best regret bound. The second version uses only
\emph{diagonal} matrices (with $d$ non-zero entries), thus trading off a
weaker bound with a better run-time in high dimensions.
Algorithm~\ref{alg:MetaGradSlave} presents the update equations in a computationally efficient form. Their intuitive motivation is given in the proof of Lemma~\ref{lem:surrogateregret}, where we show that the standard exponential weights method with Gaussian prior and  surrogate losses $\surr_t^\eta(\u)$ yields Gaussian posterior with mean $\w_t^\eta$ and covariance matrix $\Sigma_t^\eta$.
The full version of MetaGrad is closely related to the Online Newton Step
algorithm \citep{ons} running on the original losses $f_t$: the differences are that
each Slave receives the Master's gradients $\grad_t = \nabla f_t(\w_t)$
instead of its own $\nabla f_t(\w_t^\eta)$, and that an additional term $2 \eta^2 \M_t (\w_t^\eta - \w_t)$ in line~\ref{line:md} adjusts for the difference between the Slave's parameters $\w_t^\eta$ and the Master's parameters $\w_t$. MetaGrad is
therefore a bona fide first-order algorithm that only accesses $f_t$
through $\grad_t$.  We also note that we have chosen the Mirror Descent
version that iteratively updates and projects (see line~\ref{line:md}).
One might alternatively consider the Lazy Projection version (as in
\cite{Zinkevich2004,Nesterov2009,Xiao2010}) that forgets past
projections when updating on new data. Since projections are typically
computationally expensive, we have opted for the Mirror Descent version,
which we expect to project less often, since a projected point seems
less likely to update to a point outside of the domain than an
unprojected point.

\paragraph{Total run time}

As mentioned, the running time is dominated by the slaves.
Ignoring the projection, a slave with full covariance matrix takes
$O(d^2)$ time to update, while slaves with diagonal covariance matrix
take $O(d)$ time. If there are $m$ slaves, this makes the overall
computational effort respectively $O(md^2)$ and $O(md)$, both in time
per round and in memory.
Our analysis below indicates that $m = 1 + \ceil{\half \origlog_2 T}$
slaves suffice, so $m \leq 16$ as long as $T \leq 10^9$. In addition, each slave may incur the cost of a projection, which
depends on the shape of the domain $\U$. To get a sense for the
projection cost we consider a typical example. For the Euclidean ball a
diagonal projection can be performed using a few iterations of Newton's
method to get the desired precision. Each such iteration costs $O(d)$
time. This is generally considered affordable. For full projections the
story is starkly different. We typically reduce to the diagonal case by
a basis transformation, which takes $O(d^3)$ to compute using SVD. Hence
here the projection dwarfs the other run time by an order of magnitude.
We refer to \cite{adagrad} for examples of how to compute projections
for various domains $\U$. Finally, we remark that a potential speed-up
is possible by running the slaves in parallel.

\section{Analysis}\label{sec:analysis}

We conduct the analysis in three parts. We first discuss the master,
then the slaves and finally their composition. The idea is the
following. The master guarantees for all $\eta$ simultaneously that
\begin{subequations}
\label{eq:plan}
\begin{equation}\label{eq:master.plan}
0
~=~
\sum_{t=1}^T \surr_t^\eta(\w_t) 
~\le~
\sum_{t=1}^T \surr_t^\eta(\w_t^\eta) 
+ \text{master regret compared to $\eta$-slave}
.
\end{equation}
Then each $\eta$-slave takes care of learning $\u$, with regret $O(d
\log T)$:
\begin{equation}\label{eq:slave.plan}
\sum_{t=1}^T \surr_t^\eta(\w_t^\eta) 
~\le~
\sum_{t=1}^T \surr_t^\eta(\u) 
+
\text{$\eta$-slave regret compared to $\u$}
.
\end{equation}
These two statements combine to
\begin{equation}\label{eq:overall.plan}
 \eta \sum_{t=1}^T (\w_t-\u)^\top \grad_t
- \eta^2 V_T^\u
~=~
- \sum_{t=1}^T \surr_t^\eta(\u) 
~\le~
\text{sum of regrets above}
\end{equation}
\end{subequations}
and the overall result follows by optimizing $\eta$.

\subsection{Master}

To show that we can aggregate the slave predictions, we consider the
potential $\Phi_T \df \sum_\eta \pi_1^\eta e^{- \alpha \sum_{t=1}^T
\surr_t^\eta(\w_t^\eta)}$. In Appendix~\ref{app:masterProof}, we bound
the last factor $e^{- \alpha \surr_T^\eta(\w_T^\eta)}$ above by its
tangent at $\w_T^\eta = \w_T$ and obtain an objective that can be shown
to be equal to $\Phi_{T-1}$ regardless of the gradient $\grad_T$ if
$\w_T$ is chosen according to the Master algorithm. It follows that the
potential is non-increasing:
\begin{lemma}[Master combines slaves]\label{lem:pot.is.small}
The Master Algorithm guarantees $1 = \Phi_0 \ge \Phi_1 \ge \ldots \ge
\Phi_T$.
\end{lemma}
As $0 \le - \frac{1}{\alpha} \ln \Phi_T \le \sum_{t=1}^T
\surr_t^\eta(\w_t^\eta) + \frac{-1}{\alpha} \ln \pi_1^\eta$, this
implements step \eqref{eq:master.plan} of our overall proof strategy,
with master regret $\frac{-1}{\alpha} \ln \pi_1^\eta$. We further remark
that we may view our potential function $\Phi_T$ as a
\emph{game-theoretic supermartingale} in the sense of
\citeauthornumber{supermartingales}, and this lemma as establishing that
the MetaGrad Master is the corresponding \emph{defensive forecasting}
strategy.

\subsection{Slaves}

Next we implement step \eqref{eq:slave.plan}, which requires proving an
$O(d \log T)$ regret bound in terms of the surrogate loss for each
MetaGrad slave. In the full case, the surrogate loss is jointly
exp-concave, and in light of the analysis of ONS by
\citeauthornumber{ons} such a result is not surprising. For the diagonal
case, the surrogate loss lacks joint exp-concavity, but we can use
exp-concavity in each direction separately, and verify that the
projections that tie the dimensions together do not cause any trouble.
In Appendix~\ref{appx:surrogateregret} we analyze both cases
simultaneously, and obtain the following bound on the regret:

\begin{lemma}[Surrogate regret bound]\label{lem:surrogateregret}
  For $0 < \eta \leq \frac{1}{5 D G}$, let $\surr_t^\eta(\u)$ be the
  surrogate losses as defined in~\eqref{eq:surrogate} (either the full
  or the diagonal version). Then the regret of
  Slave Algorithm~\ref{alg:MetaGradSlave} is bounded by
  \begin{equation*}
    \sum_{t=1}^T \surr_t^\eta(\w_t^\eta)
      \leq \sum_{t=1}^T \surr_t^\eta(\u)
      + \frac{1}{2 D^2} \norm{\u}^2
      + \frac{1}{2} \ln \det \del*{\I + 2 \eta^2 D^2 \sum_{t=1}^T \M_t}
  \end{equation*}
for all $\u \in \U$.
\end{lemma}

\subsection{Composition}
To complete the analysis of MetaGrad, we first put the regret bounds for the master and slaves together as in \eqref{eq:overall.plan}. We then discuss how to choose the grid of $\eta$s, and optimize $\eta$ over this grid to get our main result. Proofs are postponed to Appendix~\ref{appx:composition}.

\begin{theorem}[Grid point regret]\label{thm:untuned.regret}
The full and diagonal versions of MetaGrad, with corresponding
definitions from \eqref{eq:B} and \eqref{eq:M.choices}, guarantee that,
for any grid point $\eta$ with prior weight $\pi_1^\eta$,
\[
\Rtrick_T^\u
~\le~
\eta V_T^\u
+
\frac{
  \frac{1}{2 D^2} \norm{\u}^2
  - \frac{1}{\alpha} \ln \pi_1^\eta
  + \frac{1}{2} 
  \ln \det \del*{\I + 2 \eta^2 D^2 \sum_{t=1}^T \M_t}
}{
  \eta
}
\]
for all $\u \in \U$.
\end{theorem}

\paragraph{Grid}
We now specify the grid points and corresponding prior.
Theorem~\ref{thm:untuned.regret} above implies that any two $\eta$ that
are within a constant factor of each other will guarantee the same bound
up to essentially the same constant factor. We therefore choose an
exponentially spaced grid with a heavy tailed prior (see Appendix~\ref{appx:grid}):

\begin{equation}\label{eqn:grid}
\eta_i 
~\df~
\frac{2^{-i}}{5 D G}
\quad
\text{and}
\quad
\pi_1^{\eta_i} 
~\df~
\frac{C}{(i+1)(i+2)}
\quad
\text{for $i=0,1,2,\ldots,\ceil{\half \origlog_2 T}$,}
\end{equation}
with normalization $C = 1 + \wfrac{1}{(1 + \ceil{\half \origlog_2 T})}$. At the cost of a worse
constant factor in the bounds, the number of slaves can be reduced
by using a larger spacing factor, or by omitting some of the
smallest learning rates. The net effect of \eqref{eqn:grid} is that, for
any $\eta \in
[\frac{1}{5 D G \sqrt{T}},\frac{2}{5 D G}]$ there is an $\eta_i \in
[\half \eta, \eta]$, for which $- \ln \pi_1^{\eta_i} \leq 2\log(i+2) = O(\ln \ln (1/\eta_i)) = O(\ln \ln (1/\eta))$. 
As these costs are independent of $T$, our regret guarantees still hold if the grid~\eqref{eqn:grid} is instantiated with $T$ replaced by any upper bound.

The final step is to apply
Theorem~\ref{thm:untuned.regret} to this grid, and to
properly select the learning rate $\eta_i$ in the bound. This leads to our main result:
\begin{theorem}[MetaGrad Regret Bound]\label{thm:mainbound}
Let $\S_T = \sum_{t=1}^T \M_t$ and $V_{T,i}^\u = \sum_{t=1}^T (u_i -
w_{t,i})^2 \grads_{t,i}^2$. Then the regret of
MetaGrad, with corresponding definitions from \eqref{eq:B} and
\eqref{eq:M.choices} and with grid and prior as in \eqref{eqn:grid}, is
bounded by
\begin{equation*}
  \Rtrick_T^\u ~\le~ \sqrt{ 8 V_T^\u \del*{\frac{1}{D^2} \norm{\u}^2  + \Xi_T +
  \frac{1}{\alpha}C_T}} + 5
  D G \del*{\frac{1}{D^2} \norm{\u}^2 + \Xi_T + \frac{1}{\alpha}C_T}
\end{equation*}
for all $\u \in \U$,
where
\begin{align*}
\Xi_T &\le \min \set*{\ln \det \del*{\I + \frac{D^2 \rk(\S_T)}{V_T^\u} \S_T},
\rk(\S_T) \ln \del*{\frac{D^2}{V_T^\u} \sum_{t=1}^T \|\grad_t\|^2}}
\\
&= O(d \log(D^2 G^2 T))
\end{align*}
for the full version of the algorithm,
\begin{equation*}
  \Xi_T = \sum_{i=1}^d \log \del*{\frac{D^2 \sum_{t=1}^T
  \grads_{t,i}^2}{V_{T,i}^\u}} = O(d \log(D^2 G^2 T))
\end{equation*}
for the diagonal version, and $C_T = 4 \log\del*{3 + \half \origlog_2 T}
= O(\log \log T)$ in both cases. Moreover, for both versions of the
algorithm, the regret is simultaneously bounded by
\begin{equation*}
\Rtrick_T^\u
  \leq 
  \sqrt{
8 D^2 \del*{\sum_{t=1}^T \|\grad_t\|^2}
  \del*{
  \frac{1}{D^2} \norm{\u}^2
  + \frac{1}{\alpha} C_T
}}
 + 
5 D G
  \del*{\frac{1}{D^2} \norm{\u}^2
  + \frac{1}{\alpha} C_T}
\end{equation*}
for all $\u \in \U$.
\end{theorem}

These two bounds together show that the full version of MetaGrad
achieves the new adaptive guarantee of Theorem~\ref{thm:roughthm}. The
diagonal version behaves like running the full version separately per
dimension, but with a single shared learning rate.

\section{Discussion and Future Work}

One may consider extending MetaGrad in various directions. In particular
it would be interesting to speed up the method in high dimensions, for instance by sketching \cite{SON16}.
A broader question is to identify and be adaptive to more types of easy
functions that are of practical interest. One may suspect there to be a
price (in regret overhead and in computation) for broader adaptivity,
but based on our results for MetaGrad it does not seem like we are
already approaching the point where this price is no longer worth
paying.

\paragraph{Acknowledgments}
We would like to thank Haipeng Luo and the anonymous reviewers (in particular Reviewer 6)
for valuable comments.
Koolen acknowledges support by the Netherlands Organization for Scientific Research (NWO, Veni grant 639.021.439).

\DeclareRobustCommand{\VAN}[3]{#3} %
\bibliographystyle{abbrvnat}
{\small
\bibliography{../bib}\clearpage}

\appendix

\DeclareRobustCommand{\VAN}[3]{#2} %

\section{Extra Material Related to Section~\ref{sec:fastRateExamples}}
\label{app:MoreFastRateExamplesAndProofs}

In this section we gather extra material related to the fast rate
examples from Section~\ref{sec:fastRateExamples}. We first provide
simulations. Then we present the proofs of
Theorems~\ref{thm:curvedfunctions} and~\ref{thm:Bernstein}. And finally
we give an example in which the unregularized hinge loss satisfies the
Bernstein condition.

\subsection{Simulations: Logarithmic Regret without Curvature}
\label{app:simulations}

We provide two simple simulation examples to illustrate the sufficient
conditions from Theorems~\ref{thm:curvedfunctions} and
\ref{thm:Bernstein}, and to show that such fast rates are not
automatically obtained by previous methods for general functions. Both
our examples are one-dimensional (so the full and diagonal algorithms
coincide), and have a stable optimum (that good algorithms will converge
to); yet the functions are based on absolute values, which are neither
strongly convex nor smooth, so the gradient norms do not vanish near the
optimum. As our baseline we include AdaGrad \citep{adagrad}, because it
is commonly used in practice
\citep{WordRepresentations,NeuralNetworkReview} and because, in the
one-dimensional case, it implements GD with an adaptive tuning of the
learning rate that is applicable to general convex functions.

In the first example, we consider offline convex optimization of the
fixed function $f_t(u) \equiv f(u) = \abs{u-\frac{1}{4}}$, which
satisfies \eqref{eqn:curvedfunctions}, because it is convex. In the
second example, we look at stochastic optimization with convex functions
$f_t(u) = \abs{u - x_t}$, where the outcomes $x_t = \pm \half$ are
chosen i.i.d.\ with probabilities $0.4$ and $0.6$. These probabilities
satisfy \eqref{eqn:bernstein} with $\beta = 1$. Their values are by no
means essential, as long we avoid the worst case where the probabilities
are equal.

Figure~\ref{fig:killer} graphs the results. We see that in both cases
the regret of AdaGrad follows its $O(\sqrt{T})$ bound, while MetaGrad
achieves an $O(\ln T)$ rate, as predicted by
Theorems~\ref{thm:curvedfunctions} and~\ref{thm:Bernstein}. This shows
that MetaGrad achieves a type of adaptivity that is not achieved by
AdaGrad. We should be careful in extending this conclusion to higher
dimensions, though: whereas (the diagonal version of) AdaGrad uses a
separate learning rate per dimension, MetaGrad shares learning rates
between dimensions (unless we run a separate copy of MetaGrad per
dimension, as suggested in the related work section).

\begin{figure}[t]
\centering
\subfigure[{Offline: $f_t(u) = \abs{u-1/4}$} ]{
\includegraphics[width=.45\textwidth]{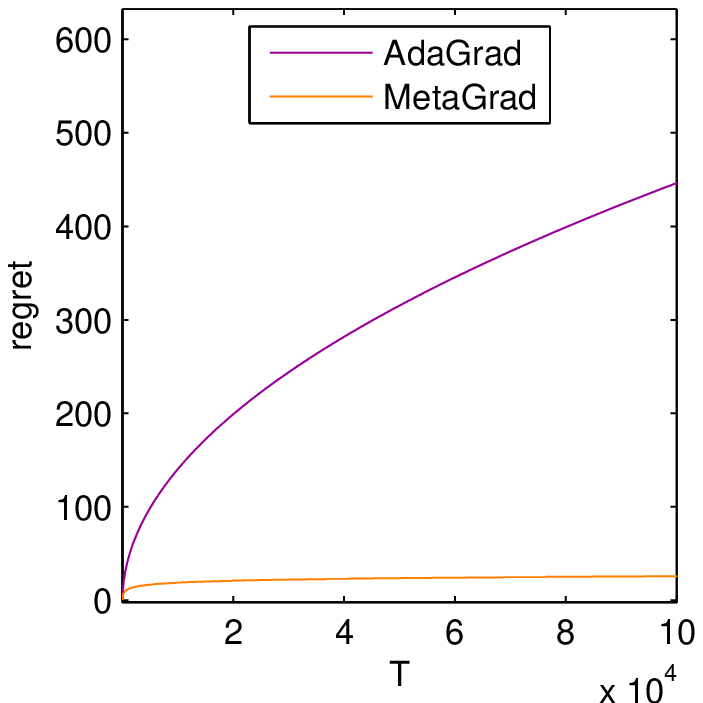}
}
\subfigure[{Stochastic Online: $f_t(u) = \abs{u- x_t}$ where $x_t = \pm
\half$ i.i.d.\ with probabilities $0.4$ and $0.6$.}]{
\includegraphics[width=.45\textwidth]{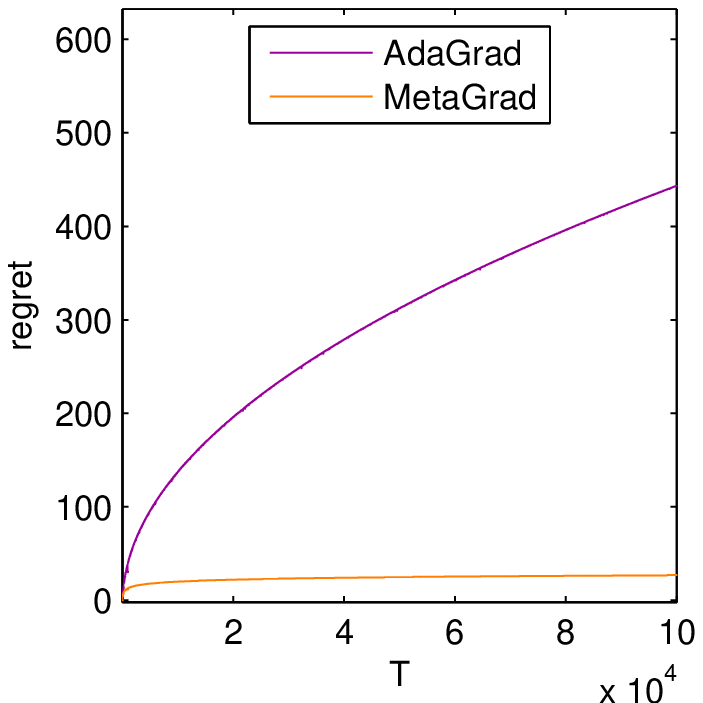}
}
\caption{Examples of fast rates on functions without curvature. MetaGrad
incurs logarithmic regret $O(\log T)$, while AdaGrad incurs
$O(\sqrt{T})$ regret, matching its bound.}
\label{fig:killer}
\end{figure}

\subsection{Proof of Theorem~\ref{thm:curvedfunctions}}

\begin{proof}
By \eqref{eqn:curvedfunctions}, applied with $\w = \w_t$, and
Theorem~\ref{thm:roughthm}, there exists a $C > 0$ (depending on $a$)
such that, for all sufficiently large $T$,
{\allowdisplaybreaks%
\begin{align*}
  R_T^\u
    &\leq a \Rtrick_T^\u - b V_T^\u
    \leq C\sqrt{ V_T^\u\, d \ln T } + C d \ln T - b V_T^\u\\
    &\leq \frac{\gamma}{2} C V_T^\u + \del*{\frac{1}{2\gamma}+1} C d \ln
    T - b V_T^\u
    \qquad \text{for all $\gamma > 0$,}
\end{align*}}
where the last inequality is based on $\sqrt{xy} = \min_{\gamma > 0}
\frac{\gamma}{2} x + \frac{y}{2\gamma}$ for all $x,y > 0$. The result
follows upon taking $\gamma = \frac{2b}{C}$.
\end{proof}

\subsection{Proof of Theorem~\ref{thm:Bernstein}}

\begin{proof}
Abbreviate $\rtrick_t^\u = (\w_t - \u)^\top \grad_t$. Then, by
\eqref{eqn:roughmainbound}, Jensen's inequality and the Bernstein
condition, there exists a constant $C > 0$ such that, for all
sufficiently large $T$, the expected linearized regret is at most
\begin{align*}
\E\sbr*{\Rtrick_T^{\u^*}} &\leq
C \E\sbr*{\sqrt{ V_T^{\u^*}\, d \ln T }} + C d \ln T
\leq C \sqrt{ \E\sbr*{V_T^{\u^*}}\, d \ln T} + C d \ln T\\
&\leq C \sqrt{ B \sum_{t=1}^T \del*{\E\sbr*{\rtrick_t^{\u^*}}}^\beta \, d \ln T} + C d \ln T.
\end{align*}
We will repeatedly use the fact that
\begin{equation}\label{eqn:linearize}
  x^\alpha y^{1-\alpha}
   = c_\alpha \inf_{\gamma > 0} \del*{\frac{x}{\gamma} +
   \gamma^\frac{\alpha}{1-\alpha}y}
  \qquad \text{for any $x,y \geq 0$ and $\alpha \in (0,1)$,}
\end{equation}
where $c_\alpha = (1-\alpha)^{1-\alpha} \alpha^\alpha$. Applying this
first with $\alpha = 1/2$, $x = B d \ln T$ and $y = \sum_{t=1}^T
\del*{\E\sbr{\rtrick_t^{\u^*}}}^\beta$, we obtain
\begin{equation*}
  \sqrt{ B \sum_{t=1}^T \del*{\E\sbr{\rtrick_t^{\u^*}}}^\beta \, d \ln T}
  \leq 
  c_{1/2}\gamma_1 \sum_{t=1}^T \del*{\E\sbr{\rtrick_t^{\u^*}}}^\beta +
  \frac{c_{1/2}}{\gamma_1} B d \ln T
  \quad \text{for any $\gamma_1 > 0$.}
\end{equation*}
If $\beta = 1$, then $\sum_{t=1}^T \del*{\E\sbr{\rtrick_t^{\u^*}}}^\beta
= \E\sbr{\Rtrick_T^{\u^*}}$ and the result follows by taking $\gamma_1 =
\frac{1}{2C c_{1/2}}$. Alternatively, if $\beta < 1$, then we apply
\eqref{eqn:linearize} a second time, with $\alpha = \beta$, $x =
\E\sbr{\rtrick_t^{\u^*}}$ and $y=1$, to find that, for any $\gamma_2>0$,
\begin{align*}
  \sqrt{ B \sum_{t=1}^T \del*{\E\sbr{\rtrick_t^{\u^*}}}^\beta \, d \ln T}
  &\leq 
  c_\beta c_{1/2} \gamma_1 \sum_{t=1}^T
  \del*{\frac{\E\sbr{\rtrick_t^{\u^*}}}{\gamma_2} +
  \gamma_2^{\beta/(1-\beta)}} +
  \frac{c_{1/2}}{\gamma_1} B d \ln T\\
  &= 
  \frac{c_\beta c_{1/2} \gamma_1}{\gamma_2} \E\sbr{\Rtrick_T^{\u^*}}
  + c_\beta c_{1/2} \gamma_1 \gamma_2^{\beta/(1-\beta)} T
  + \frac{c_{1/2}}{\gamma_1} B d \ln T.
\end{align*}
Taking $\gamma_1 = \frac{\gamma_2}{2 c_\beta c_{1/2} C}$, this yields
\begin{equation*}
  \E\sbr{\Rtrick_T^{\u^*}}
    \leq
  \gamma_2^{1/(1-\beta)} T
  + \frac{4 C^2 c_{1/2}^2 c_\beta B d \ln T}{\gamma_2} + 2 Cd\ln T.
\end{equation*}
We may optimize over $\gamma_2$ by a third application of
\eqref{eqn:linearize}, now with the choices $x = 4 C^2 c_{1/2}^2 c_\beta B d \ln T$,
$y = T$ and $\alpha = 1/(2-\beta)$, such that $\alpha/(1-\alpha) =
1/(1-\beta)$:
\begin{align*}
  \E\sbr{\Rtrick_T^{\u^*}}
    &\leq \frac{1}{c_{1/(2-\beta)}}
    \del*{4 C^2 c_{1/2}^2 c_\beta B d \ln T}^{1/(2-\beta)}
    T^{(1-\beta)/(2-\beta)}
    + 2 Cd\ln T\\
    &= O\del*{\del*{B d \ln T}^{1/(2-\beta)} T^{(1-\beta)/(2-\beta)}
    + d\ln T},
\end{align*}
which completes the proof.
\end{proof}

\subsection{Unregularized Hinge Loss Example}
\label{app:hingeLossExample}

As shown by \citeauthornumber{AndereNIPSpaper2016}, the Bernstein
condition is satisfied in the following classification task:
\begin{lemma}[Unregularized Hinge Loss Example]
  Consider i.i.d.\ labeled observations $(\X_1,Y_1),(\X_2,Y_2),\ldots$ with $Y_t$
  taking values in $\{-1,+1\}$, and let $f_t(\u) = \max\{0,1 - Y_t
  \ip{\u}{\X_t}\}$ be the \emph{hinge loss}. Assume that both $\U$ and
  the domain for $\X_t$ are the $d$-dimensional unit ball. Then the
  $(B,\beta)$-Bernstein condition is satisfied with $\beta = 1$ and $B =
  \frac{2\lambdamax}{\|\vmu\|}$, where $\lambdamax$ is the maximum
  eigenvalue of $\E\sbr*{\X \X^\top}$ and $\vmu = \E\sbr{Y\X}$, provided
  that $\|\vmu\| > 0$.

  In particular, if $\X_t$ is uniformly distributed on the sphere and the labels are 
  $Y_t = \sign(\ip{\bar{\u}}{\X_t})$, i.e.\ the noiseless classification of
  $\X_t$ according to the hyperplane with normal vector $\bar{\u}$, then $B
  \leq \frac{c}{\sqrt{d}}$ for some absolute constant $c > 0$.
\end{lemma}

Thus the version of the Bernstein condition that implies an $O(d \log T)$
rate is always satisfied for the hinge loss on the unit ball, except
when $\|\vmu\| = 0$, which is very natural to exclude, because it
implies that the expected hinge loss is $1$ (its maximal value) for all
$\u$, so there is nothing to learn. It is common to add
$\ell_2$-regularization to the hinge loss to make it strongly convex,
but this example shows that that is not necessary to get logarithmic
regret.

\section{Master Regret Bound (Proof of Lemma~\ref{lem:pot.is.small})}
\label{app:masterProof}

\begin{proof}
To prove Lemma~\ref{lem:pot.is.small}, we start by bounding $e^{- \alpha
\surr_t^\eta(\w_t^\eta)}$ by its tangent at $\w_t^\eta = \w_t$:
\begin{equation}\label{eqn:tangentbound}
e^{- \alpha \surr_t^\eta(\w_t^\eta)}
~\le~
1 + \alpha \eta \del*{\w_t - \w_t^\eta}^\top \grad_t
\qquad \text{for any $\eta \in (0,\tfrac{2}{3 D G}]$.}
\end{equation}
For the full case, where $\alpha = \alphafull = 1$, this follows
directly from the ``prod bound'' $e^{x - x^2} \leq 1+x$ with $x = \eta
\del*{\w_t - \w_t^\eta}^\top \grad_t$, which has previously been used in
the prediction with expert advice setting
\citep{CesaBianchiMansourStoltz2007,GaillardStoltzVanErven2014,squint}
and holds for any $x \geq -2/3$. In the diagonal case,
\eqref{eqn:tangentbound} does not hold with $\alpha = 1$, but it can be
proved with $\alpha = \alphadiag = 1/d$ by an application of Jensen's
inequality combined with a separate prod bound per dimension:
\begin{align*}
e^{- \alpha \surr_t^\eta(\w_t^\eta)}
&~=~
e^{
  \sum_i
  \frac{1}{d}
  \del*{
    \eta (w_{t,i} - w_{t,i}^\eta) \grads_{t,i}
    - \eta^2  (w_{t,i} - w_{t,i}^\eta)^2 \grads_{t,i}^2
  }
}
\\
&~\stackrel{\text{\tiny Jensen}}{\le}~
\sum_i
\frac{1}{d}
e^{
  \del*{
  \eta (w_{t,i} - w_{t,i}^\eta) \grads_{t,i}
    - \eta^2  (w_{t,i} - w_{t,i}^\eta)^2 \grads_{t,i}^2
  }
}
\\
&~\stackrel{\text{\tiny prod bound}}{\le}~
\sum_i
\frac{1}{d}
\del*{
1 + \eta (w_{t,i}- w^\eta_{t,i}) \grads_{t,i}
}
~=~
1 + \alpha \eta (\w_t - \w^\eta_t)^\top \grad_t
.
\end{align*}

We proceed to show that the potential $\Phi_T$ is non-increasing:
\begin{align*}
\Phi_{T+1} - \Phi_T
&~=~
\sum_\eta \pi_1^\eta e^{- \alpha\sum_{t=1}^T \surr_t^\eta(\w_t^\eta)} \del*{
  e^{- \alpha \surr_t^\eta(\w_{T+1}^\eta)}
  - 1}
\\
&~\le~
\sum_\eta \pi_1^\eta e^{- \alpha \sum_{t=1}^T \surr_t^\eta(\w_t^\eta)}
  \alpha \eta \del*{\w_{T+1} - \w_{T+1}^\eta}^\top \grad_{T+1}
~=~ 0,
\end{align*}
where the inequality is the tangent bound \eqref{eqn:tangentbound}, and
the final equality is by definition of the master prediction (in fact,
it can be taken as the motivation for the master's definition)
\[
\w_{T+1}
~=~
\frac{
  \sum_\eta \pi_{T+1}^\eta \eta \w_{T+1}^\eta
}{
  \sum_\eta \pi_{T+1}^\eta \eta \phantom{\w_{T+1}^\eta}
}
~=~
\frac{
  \sum_\eta \pi_1^\eta e^{- \alpha \sum_{t=1}^T \surr_t^\eta(\w_t^\eta)}
  \eta \w_{T+1}^\eta
}{
  \sum_\eta \pi_1^\eta e^{- \alpha \sum_{t=1}^T \surr_t^\eta(\w_t^\eta)}
  \eta
}
.
\]
Since $\Phi_0 = 1$ is trivial, this completes the proof of the lemma.
\end{proof}

\section{Slave Regret Bound (Proof of Lemma~\ref{lem:surrogateregret})}\label{appx:surrogateregret}
\begin{proof}
For any distributions $P$ and $Q$ on $\reals^d$, let $\KL(P\|Q) =
\E_P\sbr{\log \frac{\der P}{\der Q}}$ denote the Kullback-Leibler
divergence of $P$ from $Q$, and let $\vmu_P = \E_P[\u]$ denote the mean
of $P$. In addition, let $\normal(\vmu,\Sigma)$ denote a normal
distribution with mean $\vmu$ and covariance matrix $\Sigma$.
  
  In round $t$, we play according to the mean of a
  multivariate Gaussian distribution $P_t$. In the first round, this is
  a normal distribution, which plays the role of a prior:
  \begin{equation*}
    P_1 = \normal(\zeros,D^2 \I).
  \end{equation*}
  Then we update using the exponential weights update, followed by a
  projection onto $\P = \{P : \vmu_P \in \U\}$, such that the mean stays
  in the allowed domain $\U$:
  \begin{align*}
    \der \tilde{P}_{t+1}(\u) &= \frac{e^{-\surr_t^\eta(\u)}\der
    P_t(\u)}{\int_{\reals^d} e^{-\surr_t^\eta(\u')}\der P_t(\u')},
      &
    P_{t+1} &= \argmin_{P \in \P}  \KL(P\|\tilde{P}_t).
  \end{align*}
  To see that Algorithm~\ref{alg:MetaGradSlave} implements this
  algorithm, we prove by induction that
  \begin{equation*}
    P_t = \normal\del{\w_t^\eta, \Sigma_t^\eta}.
  \end{equation*}
  For $t=1$ this is clear, and if it holds for any $t$ then it can be
  verified by comparing densities that $\tilde{P}_{t+1} =
  \normal\del{\widetilde{\w}_{t+1}^\eta, \Sigma^\eta_{t+1}}$.
 Since it is well-known that the
  KL-projection of a Gaussian $\normal\del{\vmu,\Sigma}$ onto $\P$ is
  another Gaussian $\normal\del{\vnu,\Sigma}$ with the same covariance
  matrix and mean $\vnu \in \U$ that minimizes $\frac{1}{2} (\vnu -
  \vmu)^\top \Sigma^{-1} (\vnu - \vmu)$, it then follows that $P_{t+1} =
  \normal\del{\w_{t+1}^\eta, \Sigma_{t+1}^\eta}$. For completeness we
  provide a proof of this last result in Lemma~\ref{lem:projectGaussian}
  of Appendix~\ref{gausproj}.

  It now remains to bound the regret. Since $\P$ is convex, the
  Pythagorean inequality for Kullback-Leibler divergence implies that
  \begin{equation*}
    \KL\delcc{Q}{\tilde{P}_{t+1}}
    \geq \KL\delcc{Q}{P_{t+1}} + \KL\delcc{P_{t+1}}{\tilde{P}_{t+1}}
    \geq \KL\delcc{Q}{P_{t+1}}
  \end{equation*}
  for all $Q \in \P$. The following telescoping sum therefore gives us
  that
  \begin{align}
    \KL\delcc{Q}{P_1}
      &\geq \sum_{t=1}^T \KL\delcc{Q}{P_t} - \KL\delcc{Q}{P_{t+1}}
      \geq \sum_{t=1}^T \KL\delcc{Q}{P_t} -
      \KL\delcc{Q}{\tilde{P}_{t+1}}\notag \\
      &= \sum_{t=1}^T -\log \E_{P_t}\sbr{e^{-\surr_t^\eta(\u)}}
      - \E_Q\sbr{\surr_t^\eta(\u)}.\label{eqn:Bregmantelescope}
  \end{align}
  This may be interpreted as a regret bound in the space of
  distributions, which we will now relate to our regret of interest. If
  $\M_t = \Mfull_t$, then Lemma~\ref{lem:gauss.exp.concave} in
  Appendix~\ref{app:gauss.exp.concave} implies that
  \begin{equation*}
    -\log \E_{P_t}\sbr{e^{-\surr_t^\eta(\u)}}
    \geq \surr_t^\eta(\w_t^\eta)
  \end{equation*}
  because $\w_t^\eta$ is the mean of $P_t$. Alternatively, if $\M_t =
  \Mdiag_t$ then $P_t$ has diagonal covariance $\Sigma_t^\eta$, and we
  can use Lemma~\ref{lem:gauss.exp.concave} again to draw the same conclusion.
  
  To control $\E_Q\sbr{\surr_t^\eta(\u)}$, we may restrict attention (without loss of generality by a standard maximum entropy argument) to
  normal distributions $Q = \normal\del{\vmu, D^2 \Sigma}$ with mean $\vmu
  \in \U$ and covariance $\Sigma \succ \zeros$ (expressed relative to the prior variance $D^2$). Then, using the
  cyclic property and linearity of the trace,
  \begin{align*}
    \E_Q\sbr{\surr_t^\eta(\u)} 
      &= - \eta(\w_t-\vmu)^\top \grad_t + \eta^2
      \del{\w_t^\top \M_t \w_t - 2\vmu^\top \M_t \w_t +
      \E_Q\sbr{\Tr\del{\u^\top \M_t
      \u}}}\\
      &= - \eta(\w_t-\vmu)^\top \grad_t + \eta^2
      \del{\w_t^\top \M_t \w_t - 2\vmu^\top \M_t \w_t + \Tr\del{\E_Q\sbr{\u
      \u^\top}\M_t}}\\
      &= - \eta(\w_t-\vmu)^\top \grad_t + \eta^2
      \del{\w_t^\top \M_t \w_t - 2\vmu^\top \M_t \w_t +
      \Tr\del{\del{D^2 \Sigma + \vmu \vmu^\top}\M_t}}
    \\
      &= - \eta(\w_t-\vmu)^\top \grad_t + \eta^2
      \del*{
        (\vmu - \w_t)^\top \M_t (\vmu - \w_t)
        +  D^2 \Tr\del{ \Sigma \M_t}
        }
    \\
      &= \surr_t^\eta(\vmu)  + \eta^2 D^2
         \Tr\del{ \Sigma \M_t}.
  \end{align*}
Finally, it remains to work out
\[
\KL\delcc{Q}{P_1}
~=~
\frac{1}{2 D^2} \norm{\vmu}^2
+
\frac{1}{2} \del*{
-  \ln \det \Sigma
  + \tr(\Sigma) - d
}
.
\]
We have now bounded all the pieces in \eqref{eqn:Bregmantelescope}.
Putting them all together with the choice $\vmu = \u$ and optimizing the
bound in $\Sigma$ gives:
\begin{align}
\notag
&
  \sum_{t=1}^T \surr_t^\eta(\w_t^\eta)
  - \sum_{t=1}^T \surr_t^\eta(\u)
\\
\notag
&~\le~
\frac{1}{2 D^2} \norm{\u}^2
+
\frac{1}{2}
\inf_{\Sigma \succ \zeros}~ \set*{
  -  \ln \det \Sigma
  + \tr \del*{ \Sigma \del*{\I + 2 \eta^2 D^2 \sum_{t=1}^T \M_t}} - d
}
\\
\label{eq:variational.sigma}
&~=~
\frac{1}{2 D^2} \norm{\u}^2
+ \frac{1}{2} \ln \det \del*{\I + 2 \eta^2 D^2 \sum_{t=1}^T \M_t}
,
\end{align}
where the minimum is attained at $\Sigma
=
\del*{
  \I
  + 2 \eta^2 D^2 \sum_{t=1}^T \M_t 
}^{-1}$.
\end{proof}

\section{Composition Proofs}\label{appx:composition}
Throughout this section we abbreviate $\S_T = \sum_{t=1}^T \M_t$. 

\subsection{Proof of Theorem~\ref{thm:untuned.regret}}
\begin{proof}
We start with
\begin{multline*}
0
~\stackrel{\text{\tiny Lemma~\ref{lem:pot.is.small}}}\ge~
\frac{1}{\alpha} \ln \Phi_T
~\ge~
\frac{1}{\alpha} \ln \pi_1^\eta
- \sum_{t=1}^T \surr_t^\eta(\w_t^\eta)
\\
~\stackrel{\text{\tiny Lemma~\ref{lem:surrogateregret}}}\ge~
\frac{1}{\alpha} \ln \pi_1^\eta
- \sum_{t=1}^T \surr_t^\eta(\u)
- \frac{1}{2 D^2} \norm{\u}^2
- \frac{1}{2} \ln \det \del*{\I + 2 \eta^2 D^2 \S_T}.
\end{multline*}
Now expanding the definition \eqref{eq:surrogate} of the surrogate losses we find
\[
\eta \sum_{t=1}^T (\w_t-\u)^\top \grad_t
~\le~
\frac{1}{2 D^2} \norm{\u}^2
- \frac{1}{\alpha} \ln \pi_1^\eta
+ \eta^2  V_T^\u
+ \frac{1}{2} \ln \det \del*{\I + 2 \eta^2 D^2 \S_T},
\]
in which we may divide by $\eta$ to obtain the claim.
\end{proof}

\subsection{Proof of Theorem~\ref{thm:mainbound}}

\begin{proof}
In principle we would like to directly select the $\eta$ that
optimizes the regret bound from Theorem~\ref{thm:untuned.regret}. But
unfortunately we cannot tractably minimize that bound, since $\eta$ occurs in the $\ln \det$. To bring the $\eta$ out, we apply the variational form from \eqref{eq:variational.sigma} to Theorem~\ref{thm:untuned.regret} to obtain
\begin{equation}\label{eqn:firstRbound}
\Rtrick_T^\u
\le \inf_{\Sigma \succ \zeros}~ 
\eta_i \del*{V_T^\u
  + D^2 \tr \del*{ \Sigma \S_T}
}
+
\frac{
  \frac{1}{D^2} \norm{\u}^2
  - \frac{2}{\alpha} \ln \pi_1^{\eta_i}
  - \ln \det(\Sigma)
  + \tr \del*{\Sigma}
  - d
}{
  2\eta_i
}
\end{equation}
for all grid points $\eta_i$. This leads to an upper bound by plugging
in a near-optimal choice for $\Sigma$, which we choose as
\begin{align*}
   & \text{full} & & \text{diag}\\
  \Sigma = &\del*{\I + c\S_T}^{-1}\kern-1em
  &\Sigma = \frac{1}{D^2} \diag(&V_{T,1}^\u, \ldots, V_{T,d}^\u) \S_T^{-1},
\end{align*}
where $c \df \rk(\S_T) \del*{\frac{D^2}{V_T^\u} - \frac{1}{\tr(\S_T)}}$
is non-negative because by Cauchy-Schwarz $V_T^\u \le D^2 \sum_{t=1}^T \|\grad_t\|^2 = D^2
\tr(\S_T)$. We proceed to bound terms involving
$\Sigma$ above. In the diagonal case, we use that $D^2 \tr \del*{ \Sigma
\S_T} = V_T^\u$ and $\Sigma \prec \I$ because $V_{T,i}^\u \leq
D^2 \sum_{t=1}^T \grads_{t,i}^2$ by Cauchy-Schwarz. In the full case, we
also have $\Sigma \prec \I$. In addition,
we observe that $\S_T$ and $\Sigma$ share the same eigenbasis, so we may
work in that basis. As $\Sigma \S_T$ has $\rk(\S_T)$ non-zero
eigenvalues, we may pull out a factor $\rk(\S_T)$ and replace the trace
by a uniform average of the eigenvalues. Then Jensen's inequality for
the concave function $x \mapsto \frac{x}{1 + c x}$ for $x \ge 0$ gives
\begin{align*}
D^2 \tr(\Sigma \S_T) 
&~\stackrel{\text{\tiny Jensen}}{\le}~
\frac{
  D^2 \tr(\S_T)
}{
  \del*{1 + \frac{c}{\rk(\S_T)} \tr(\S_T)}
}
~=~
V_T^\u.
\end{align*}
Thus, in both cases we have $D^2 \tr(\Sigma \S_T) \leq V_T^\u$ and
 $\Sigma \prec \I$, which implies that $\tr(\Sigma) ~\le~ \tr(\I)
= d$ and that
\begin{equation*}
  \Xi_T \df - \ln \det \Sigma \geq 0.
\end{equation*}
Finally, by construction of the grid, for any $\eta \in [\frac{1}{5 D G
\sqrt{T}},\frac{2}{5 D G}]$ there exists a grid point $\eta_i \in
[\frac{\eta}{2},\eta]$, and the prior costs of this grid point satisfy
\begin{equation*}
  -\log \pi_1^{\eta_i}
    \leq 2\log(2+i)
    \leq 2\log\del*{3 + \half \origlog_2 T}.
\end{equation*}

Plugging these bounds into \eqref{eqn:firstRbound} and abbreviating
\begin{align*}
A &\df \frac{1}{D^2} \norm{\u}^2 + \frac{4}{\alpha} \log\del*{3 +
\half \origlog_2 T} + \Xi_T \geq 4 \log 3,
\end{align*}
we obtain
\begin{equation*}
\Rtrick_T^\u
~\le~ 
2 \eta V_T^\u
+
\frac{A}{
  \eta
}.
\end{equation*}
Subsequently tuning $\eta$ optimally as
\[
\hat\eta
~=~
\sqrt{
  \frac{A
  }{
    2 V_T^\u
  }
} \geq \frac{\sqrt{2 \log
3}}{D G \sqrt{T}}
  \geq \frac{1}{5 D G \sqrt{T}}
\]
is allowed when $\hat\eta \leq \frac{2}{5 D G}$, and gives $\Rtrick_T^\u ~\le~
\sqrt{8 V_T^\u A }$. Alternatively, if $\hat\eta \geq
\frac{2}{5 D G}$, then we plug in $\eta = \frac{2}{5 D G}$ and obtain
$\Rtrick_T^\u
\le
\frac{4}{5 D G} V_T^\u
+
\frac{5}{2} D G A
~\le~ 
5 D G A
$, where the second inequality follows from the constraint on $\hat
\eta$.
In both cases, we find that
\begin{equation*}
  \Rtrick_T^\u ~\le~ \sqrt{8 V_T^\u A} + 5 D G A,
\end{equation*}
which results in the first claim of the theorem upon observing that, for
the full version of the algorithm, $\Xi_T \le \rk(\S_T) \ln
\del*{\frac{D^2 \tr(\S_T)}{V_T^\u} }$ by Jensen's inequality and $\Xi_T
\le \ln \det \del*{\I + \frac{D^2 \rk(\S_T)}{V_T^\u} \S_T}$ by
monotonicity of $\ln \det$.

To prove the second claim, we instead take the comparator covariance
$\Sigma = \I$ equal to the prior covariance and again use $V_T^\u \leq
D^2 \tr(\S_T)$ to find
\begin{align*}
\Rtrick_T^\u
&~\le~ 
\eta_i \del*{V_T^\u
  + D^2 \tr \del*{ \S_T}
}
+
\frac{
  \frac{1}{D^2} \norm{\u}^2
  - \frac{2}{\alpha} \ln \pi_1^{\eta_i}
}{
  2\eta_i
}
\\
&~~\le~
2 \eta_i D^2 \tr \del*{\S_T}
+
\frac{
  \frac{1}{D^2} \norm{\u}^2
  - \frac{2}{\alpha} \ln \pi_1^{\eta_i}
}{
  2\eta_i
}\\
&~\le~
2 \eta D^2 \tr \del*{\S_T}
+
\frac{
  \frac{1}{D^2} \norm{\u}^2
  + \frac{4}{\alpha} \log\del*{3 + \half \origlog_2 T}
}{
  \eta
}
\end{align*}
for all $\eta \in [\frac{1}{5 D G \sqrt{T}},\frac{2}{5 D G}]$. Tuning 
$\eta$ as
\begin{equation*}
  \hat \eta = \sqrt{
  \frac{
    \frac{1}{D^2} \norm{\u}^2
    + \frac{4}{\alpha} \log\del*{3 + \half \origlog_2 T}
  }{
    2 D^2 \tr \del*{\S_T}
  }
  }
  \geq
  \sqrt{
  \frac{
    4 \log 3
  }{
    2 D^2 G^2 T
  }
  }
  \geq
  \frac{1}{5 D G \sqrt{T}}
\end{equation*}
is allowed when $\hat \eta \leq \frac{2}{5 D G}$, and gives 
\begin{equation*}
\Rtrick_T^\u \leq 
  \sqrt{
8 D^2 \tr \del*{\S_T}
  \del*{
  \frac{1}{D^2} \norm{\u}^2
  + \frac{4}{\alpha} \log\del*{3 + \half \origlog_2 T}
}}.
\end{equation*}
Alternatively, if $\hat\eta \geq \frac{2}{5 D G}$, then we plug in $\eta
= \frac{2}{5 D G}$ and obtain
\begin{align*}
\Rtrick_T^\u
  &~\le~
    \frac{4}{5 D G} D^2 \tr \del*{\S_T}
    +
    \frac{5}{2} D G
      \del*{\frac{1}{D^2} \norm{\u}^2
      + \frac{4}{\alpha} \log\del*{3 + \half \origlog_2 T}}\\
  &~\le~
    5 D G
      \del*{\frac{1}{D^2} \norm{\u}^2
      + \frac{4}{\alpha} \log\del*{3 + \half \origlog_2 T}},
\end{align*}
where the second inequality follows from the constraint on $\hat \eta$.
In both cases, the second claim of the theorem follows.
\end{proof}

\section{Discussion of the Choice of Grid Points and Prior Weights}\label{appx:grid}

We now think about the choice of the grid and corresponding prior.
Theorem~\ref{thm:untuned.regret} above implies that any two $\eta$ that
are within a constant factor of each other will guarantee the same bound
up to a constant factor. Since $\eta$ is a continuous parameter, this suggests choosing a prior that is approximately uniform for $\log \eta$, which means it should have a
density that looks like $1/\eta$. Although
Theorem~\ref{thm:untuned.regret} does not show it, there is never any
harm in taking too many grid points, because grid points that are very
close together will behave as a single point with combined prior mass.
If we disregard computation, we would therefore like to use the prior discussed by
\citeauthornumber{ChernovVovk2010}, which is very close to uniform on $\log \eta$
and has density
\[
\pi(\eta) ~=~ \frac{C}{\eta \origlog_2^2(\frac{5}{2} D G \eta)},
\]
where we include the factor $\frac{5}{2} D G$ to make the
prior invariant under rescalings of the problem, and $C$ is a
normalizing constant that makes the prior integrate to $1$. To adapt this prior to a discrete grid, we need to integrate this density between grid points and assign prior masses:
\begin{equation*}
\pi_1^{\eta_i} 
~\df~
\int_{\eta_{i+1}}^{\eta_{i}} 
\pi(\eta) \dif \eta 
~=~
\left.
\frac{C \log(2)}{- \origlog_2(\frac{5}{2} D G \eta)}
\right|_{\eta_{i+1}}^{\eta_i}
.
\end{equation*}
For the exponentially spaced grid in \eqref{eqn:grid}, this evaluates to the prior weights $\pi_1^{\eta_i}$ specified there.

\section{Projection of Gaussians}\label{gausproj}

It is well-known that the projection of a Gaussian onto the
set of distributions with mean in the convex set $\U$ is also a Gaussian
with the same covariance matrix. This result follows easily from, for
instance, Theorems~1.8.5 and 1.8.2 of \citeauthornumber{ihara1993information}, but
we include a short proof for completeness:
\begin{lemma}\label{lem:projectGaussian}
  Let $\tilde{P}_t = \normal(\vmu,\Sigma)$ be Gaussian and let \[P_t =
  \argmin_{P \colon \vmu_P \in \U} \KL\delcc{P}{\tilde{P}_t}\] be its
  projection onto the set of distributions with mean in $\U$. Then $P_t$
  is also Gaussian with the same covariance matrix:
  \begin{equation*}
    P_t = \normal(\vnu,\Sigma)     
  \end{equation*}
  for $\vnu \in \U$ that minimizes $\frac{1}{2} (\vnu - \vmu)^\top
  \Sigma^{-1} (\vnu - \vmu)$.
\end{lemma}

\begin{proof}
  Let $P$ be an arbitrary distribution with mean $\vnu \in \U$, and let $R =
  \normal\del{\vnu,\Sigma}$. Then by straight-forward algebra and
  nonnegativity of Kullback-Leibler divergence it can be verified that
  \begin{equation*}
    \KL\delcc{P}{\tilde{P}_t}
      = \KL\delcc{P}{R} + \KL\delcc{R}{\tilde{P}_t}
      \geq \KL\delcc{R}{\tilde{P}_t}.
  \end{equation*}
  Thus the minimum over all $P$ is achieved by a Gaussian with the same
  covariance matrix as $\tilde{P}_t$. It remains to find the mean of the
  projection, which is the $\vnu \in \U$ that minimizes
  \begin{equation*}
    \KL\delcc{R}{\tilde{P}_t}
      = \frac{1}{2} (\vnu - \vmu)^\top \Sigma^{-1} (\vnu - \vmu),
  \end{equation*}
  as required.
\end{proof}

\section{Gaussian Exp-concavity}\label{app:gauss.exp.concave}

Exp-concavity of $\surr_t^\eta(\u)$ means that $\ex \sbr*{ e^{-
\surr_t^\eta(\u)} } \le e^{- \surr_t^\eta(\vmu)}$ for any distribution
on $\u \in \reals^d$. Although this does not hold for general
distributions with support outside of $\U$, it does hold if we restrict
attention to certain types of Gaussians:
\begin{lemma}[Gaussian exp-concavity]\label{lem:gauss.exp.concave}
Let $0 < \eta \leq \frac{1}{5 D G}$.
Consider a Gaussian distribution with mean $\vmu \in \U$ and arbitrary covariance
$\Sigma \succ \zeros$ in the full case or diagonal $\Sigma \succ \zeros$
in the diagonal case. Then
\[
\ex_{\u \sim \normal(\vmu, \Sigma)} \sbr*{
  e^{- \surr_t^\eta(\u)}
}
~\le~
e^{- \surr_t^\eta(\vmu)}
\]
\end{lemma}

\begin{proof}
We first consider the full case.
Abbreviating $r \df (\w_t - \vmu)^\top \grad_t$ and $s \df (\vmu -
\u)^\top \grad_t$, from the definition \eqref{eq:surrogate} of
$\surr_t^\eta$ we get
\begin{align*}
&
\surr_t^\eta(\vmu)- \surr_t^\eta(\u)
\\
&~=~
\eta(\vmu-\u)^\top \grad_t
- \eta^2 \del*{
  2(\vmu - \w_t)^\top \grad_t \grad_t^\top  (\vmu-\u)
  + (\vmu-\u)^\top \grad_t \grad_t^\top  (\vmu-\u)
}
\\
&~=~
\eta s
- \eta^2 \del*{2 r s + s^2}.
\end{align*}
Since $\u \sim \normal(\vmu, \Sigma)$ implies $s \sim \normal(0, v)$ with $v= \grad_t^\top \Sigma \grad_t$, the claim collapses to
\[
1
~\ge~
\ex_{\u \sim \normal(\vmu, \Sigma)} \sbr*{
  e^{\surr_t^\eta(\vmu)- \surr_t^\eta(\u)}
}
~=~
\ex_{s \sim \normal(0, v)} \sbr*{
  e^{
    \eta s
    - \eta^2 \del*{
       2 r s + s^2
    }
  }
}
~=~
\frac{e^{\frac{\eta ^2 v (1 - 2  \eta  r)^2}{2 (1 + 2  \eta^2 v)}}}{\sqrt{1 + 2  \eta ^2 v}}
,
\]
which is equivalent to
\[
(1 - 2 \eta r)^2 \eta^2 v
~\le~
\del*{1 + 2 \eta^2 v} \ln \del*{1 + 2 \eta^2 v}
.
\]
The left-hand side is maximized over $r \in [-\Dfull \Gfull,\Dfull \Gfull]$ at the point
$r=-\Dfull \Gfull$. So it suffices to establish
\[
v (1 + 2 \eta \Dfull \Gfull)^2 \eta^2
~\le~
\del*{1 + 2 \eta^2 v} \ln \del*{1 + 2 \eta^2 v}
.
\]
Now the right-hand is convex in $v$ and hence bounded below by its
tangent at $v=0$, which is $2 \eta^2 v$. The proof is completed by
observing that $(1 + 2 \eta \Dfull \Gfull)^2 \le 2$ by the assumed
bound on $\eta$.

It remains to consider the diagonal case. There
 the surrogate loss \eqref{eq:surrogate} is a sum over dimensions, say $\surr_t^\eta(\u) = \sum_{i=1}^d \surr_{t,i}^\eta(u_i)$. For a Gaussian with diagonal covariance matrix $\Sigma$ the coordinates of $\u$ are independent, and hence
\[
\ex_{\u \sim \normal(\vmu, \Sigma)} \sbr*{
  e^{- \surr_t^\eta(\u)}
}
~=~
\prod_{i=1}^d
\ex_{u_i \sim \normal(\mu_i, \oldSigma_{i,i})} \sbr*{
  e^{- \surr_{t,i}^\eta(u_i)}
}
~\le~
\prod_{i=1}^d e^{- \surr_{t,i}^\eta(\mu_i)}
~=~
e^{- \surr_t^\eta(\vmu)}
,
\]
where the inequality is the result for the full case applied to each
dimension separately.
\end{proof}

\section{Bernstein for Linearized Excess Loss}\label{sec:bnst}
Let $f : \U \to \reals$ be a convex function drawn from distribution $\pr$ with stochastic optimum $\u^* = \argmin_{\u \in \U} \E_{f \sim
\pr}[f(\u)]$. For any $\w \in \U$, we now show that the Bernstein condition for the excess loss $X \df f(\w)-f(\u^*)$ implies the Bernstein condition with the same exponent $\beta$ for the linearized excess loss $Y \df (\w-\u^*)^\top \nabla f(\w)$. These variables satisfy $Y \ge X$ by convexity of $f$ and $Y \le C \df \Dfull \Gfull$. 

\begin{lemma}
For $\beta \in (0,1]$, let $X$ be a $(B,\beta)$-Bernstein random variable:
\[
\ex[X^2] \le B \ex[X]^\beta.
\]
Then any bounded random variable $Y \le C$ with $Y \ge X$ pointwise satisfies the $(B',\beta)$-Bernstein condition
\[
\ex[Y^2] \le B' \ex[Y]^\beta
\]
for $B' = \max \set*{
B,
\frac{2}{\beta}
C^{2-\beta}
}
$.
\end{lemma}

\begin{proof}
For $\beta \in (0,1)$ we will use the fact that
\[
  z^\beta
  ~=~ 
  c_\beta \inf_{\gamma > 0} \del*{\frac{z}{\gamma} +
   \gamma^\frac{\beta}{1-\beta}}
  \qquad \text{for any $z \geq 0$,}
\]
with $c_\beta =   (1-\beta)^{1-\beta} \beta^\beta$.
For $\gamma = \del*{\frac{1-\beta}{\beta} \ex[Y]}^{1-\beta}$ we therefore have
\begin{align}
\notag
\ex[X^2] - B' \ex[X]^\beta
&~\ge~
\ex[X^2] - B' c_\beta \del*{\frac{\ex [X]}{\gamma}  + \gamma^{\frac{\beta}{1-\beta}}}
\\
\notag
&~\ge~
\ex[Y^2] - B' c_\beta \del*{\frac{\ex [Y]}{\gamma}  + \gamma^{\frac{\beta}{1-\beta}}}
\\
\label{eq:tolimitme}
&~=~
\ex[Y^2] - B' \ex[Y]^\beta
,
\end{align}
where the second inequality holds because $x^2 - c_\beta B' x/\gamma$ is a decreasing function of $x \le C$ for $\gamma  \le \frac{c_\beta B'}{2 C}$, which is satisfied by the choice of $B'$.
This proves the lemma for $\beta \in (0,1)$. The claim for $\beta=1$ follows by taking the limit $\beta \to 1$ in \eqref{eq:tolimitme}.
\end{proof}

\end{document}